\documentclass[english]{article}
\usepackage[T1]{fontenc}
\usepackage[utf8]{inputenc}
\usepackage{babel}
\usepackage{amsmath}
\usepackage{amsthm}
\usepackage{amssymb}
\usepackage{graphicx}
\usepackage{microtype}
\usepackage[unicode=true,
 bookmarks=false,
 breaklinks=false,pdfborder={0 0 1},backref=section,colorlinks=false]
 {hyperref}

\makeatletter
\theoremstyle{plain}
\newtheorem{thm}{\protect\theoremname}
\theoremstyle{definition}
\newtheorem{defn}[thm]{\protect\definitionname}
\theoremstyle{plain}
\newtheorem{cor}[thm]{\protect\corollaryname}
\theoremstyle{plain}
\newtheorem{lem}[thm]{\protect\lemmaname}

\usepackage{babel}

     \usepackage[final,nonatbib]{neurips_2020}

\usepackage{url}
\usepackage{booktabs}
\usepackage{amsfonts}
\usepackage{nicefrac}

\usepackage{float}

\title{How Many Samples is a Good Initial Point Worth?}

\author{%
  Gavin Zhang\ \\
  Department of Electrical and Computer Engineering\\
  University of Illinois Urbana Champaign\\
  Illinois, IL61820 \\
  \texttt{jialun2@illinois.edu} \\
  \And
   Richard Y. Zhang\ \\
  Department of Electrical and Computer Engineering\\
  University of Illinois Urbana Champaign\\
  Illinois, IL61820 \\
  \texttt{ryz@illinois.edu} \\
}

\theoremstyle{plain}

\providecommand{\corollaryname}{Corollary}
\providecommand{\definitionname}{Definition}
\providecommand{\theoremname}{Theorem}

\providecommand{\lemmaname}{Lemma}

\@ifundefined{showcaptionsetup}{}{%
 \PassOptionsToPackage{caption=false}{subfig}}
\usepackage{subfig}
\makeatother

\providecommand{\corollaryname}{Corollary}
\providecommand{\definitionname}{Definition}
\providecommand{\lemmaname}{Lemma}
\providecommand{\theoremname}{Theorem}

\begin{document}
\title{How Many Samples is a Good Initial Point Worth in Low-rank Matrix
Recovery?}

\maketitle
\global\long\def\M{\mathcal{M}}%
\global\long\def\R{\mathbb{R}}%
\global\long\def\AA{\mathcal{A}}%
\global\long\def\A{\mathbf{A}}%
\global\long\def\orth{\mathrm{orth}}%
\global\long\def\tr{\mathrm{tr}}%
\global\long\def\rank{\mathrm{rank}}%
\global\long\def\LMI{\mathrm{LMI}}%
\global\long\def\HH{\mathcal{H}}%
\global\long\def\X{\mathbf{X}}%
\global\long\def\H{\mathbf{H}}%
\global\long\def\U{\mathbf{U}}%
\global\long\def\mat{\mathrm{mat}}%
\global\long\def\e{\mathbf{e}}%
\global\long\def\vector{\mathrm{vec}\,}%
\global\long\def\ub{\mathrm{ub}}%
\global\long\def\lb{\mathrm{lb}}%
\global\long\def\sp{\mathrm{span}}%
\global\long\def\range{\mathrm{range}}%
\global\long\def\diag{\mathrm{diag}}%
\global\long\def\soc{\mathrm{soc}}%
\global\long\def\foc{\mathrm{foc}}%
\newtheorem{manualtheoreminner}{Theorem} \newenvironment{manualtheorem}[1]{%
  \renewcommand\themanualtheoreminner{#1}%
  \manualtheoreminner
}{\endmanualtheoreminner}

\newtheorem{manualleminner}{Lemma} \newenvironment{manuallem}[1]{%
  \renewcommand\themanualleminner{#1}%
  \manualleminner
}{\endmanualleminner}
\begin{abstract}
Given a sufficiently large amount of labeled data, the non-convex
low-rank matrix recovery problem contains no spurious local minima,
so a local optimization algorithm is guaranteed to converge to a global
minimum starting from any initial guess. However, the actual amount
of data needed by this theoretical guarantee is very pessimistic,
as it must prevent spurious local minima from existing anywhere, including
at adversarial locations. In contrast, prior work based on good initial
guesses have more realistic data requirements, because they allow
spurious local minima to exist outside of a neighborhood of the solution.
In this paper, we quantify the relationship between the quality of
the initial guess and the corresponding reduction in data requirements.
Using the restricted isometry constant as a surrogate for sample complexity,
we compute a sharp “threshold” number of samples needed to prevent
each specific point on the optimization landscape from becoming a
spurious local minimum. Optimizing the threshold over regions of the
landscape, we see that for initial points around the ground
truth, a linear improvement in the quality of the initial guess amounts
to a constant factor improvement in the sample complexity. 
\end{abstract}

\section{Introduction }

A perennial challenge in non-convex optimization is the possible existence
of \emph{bad} or \emph{spurious} critical points and local minima,
which can cause a local optimization algorithm like gradient descent
to slow down or get stuck. Several recent lines of work showed that
the effects of non-convexity can be tamed through a large amount of
diverse and high quality training data~\cite{sun2015complete,bhojanapalli2016dropping,ge2016matrix,boumal2016non,sun2018geometric,li2019non}.
Concretely, these authors showed that, for classes of problems based
on random sampling, spurious critical points and local minima become
progressively less likely to exist with the addition of each new sample.
After a \emph{sufficiently large number of samples}, all spurious
local minima are eliminated, so any local optimization algorithm is
guaranteed to converge to the globally optimal solution starting from
an arbitrary, possibly random initial guess.

This notion of a \emph{global} guarantee---one that is valid starting
from any initial point---is considerably stronger than what is needed
for empirical success to be observed~\cite{chi2019nonconvex}. For
example, the existence of a spurious local minimum may not pose an
issue if gradient descent does not converge towards it. However, a
theoretical guarantee is no longer possible, as starting the algorithm
from the spurious local minimum would result in failure~\cite{zhang2018much}.
As a consequence, these global guarantees tend to be pessimistic,
because the number of samples must be sufficiently large to eliminate
spurious local minima everywhere, even at adversarial locations. By
contrast, the weaker notion of a \emph{local} guarantee~\cite{keshavan2010matrix,jain2013low,netrapalli2014non,sun2016guaranteed,candes2015phase,chen2015solving,tu2016low,li2019rapid}---one
that is valid only for a specified set of initial points---is naturally
less conservative, as it allows spurious local minima to exist outside
of the specified set.

In this paper, we provide a unifying view between the notions of the
global and local guarantees by quantifying the relationship between
the sample complexity and the quality of the initial point. We restrict
our attention to the \emph{matrix sensing} problem, which seeks to
recover a rank-$r$ positive semidefinite matrix $M^{*}=ZZ^{T}\in\R^{n\times n}$
with $Z\in\R^{n\times r}$ from $m$ sub-Gaussian linear measurements
of the form 
\begin{equation}
b\equiv\mathcal{A}(ZZ^{T})\equiv\left[\left\langle A_{1},M^{*}\right\rangle \quad\cdots\quad\left\langle A_{m},M^{*}\right\rangle \right]^{T}\label{A}
\end{equation}
by solving the following non-convex optimization problem: 
\begin{equation}
\min_{X\in\mathbb{R}^{n\times r}}f_{\AA}(X)\equiv\left\Vert \mathcal{A}\left(XX^{T}-ZZ^{T}\right)\right\Vert ^{2}=\sum_{i=1}^{m}\left(\left\langle A_{i},XX^{T}\right\rangle -b_{i}\right)^{2}.\label{p1}
\end{equation}
We characterize a sharp “threshold” on the number of samples $m$
needed to prevent each specific point on the optimization landscape
from becoming a spurious local minimum. While the threshold is difficult
to solve, we derive a lower-bound in closed-form based on spurious
\emph{critical points}, and show that it constitutes a \emph{sharp}
lower-bound on the original threshold of interest. The lower-bound
reveals a simple geometric relationship: a point $X$ is more likely
to be a local minimum if the column spaces of $X$ and $Z$
are close to orthogonal. Optimizing the closed-form lower-bound over
regions of the landscape, we show that for initial points close
to the ground truth, a constant factor improvement of the initial
point amounts to a constant factor reduction in the number of samples
needed to guarantee recovery.

\section{Related Work}

\textbf{Local Guarantees}. The earliest work on exact guarantees for
non-convex optimization focused on generating a good initial guess
within a local region of attraction. For instance, in~\cite{tu2015low,zheng2015convergent},
the authors showed that when $\AA$ satisfies $(\delta,6r)$-RIP with a constant
$\delta\leq1/10$, and there exists a initial point sufficiently close
to the ground truth, then gradient descent starting from this initial point
has a linear convergence rate. The typical strategy to find such the
initial point is \textit{spectral initialization}~\cite{keshavan2010matrix,jain2013low,tu2015low,sun2016guaranteed,candes2015phase,ma2019implicit,chen2015fast}:
using the singular value decomposition on a surrogate matrix to find
low-rank factors that are close to the ground truth.

In this paper, we focus on the trade-off between the quality of an
initial point and the number of samples needed to prevent the existence
of spurious local minima, while sidestepping the question of how it
is found. We note, however, that the number of samples needed to find
an $\epsilon$-good initial guess (e.g. via spectral initialization)
forms an interesting secondary trade-off. It remains a future work
to study the interactions between these two points.

\textbf{Global Guarantees}. Recent work focused on establishing a
global guarantee that is independent of the initial guess~\cite{sun2015complete,bhojanapalli2016dropping,ge2016matrix,boumal2016non,sun2018geometric,li2019non}.
For our purposes, Bhojanapalli et al. \cite{bhojanapalli2016global}
showed that RIP with $\delta_{2r}<1/5$ eliminates all spurious local
minima, while Zhang et al. \cite{zhang2019sharp} refined this to
$\delta_{2r}<1/2$ for the rank-1 case, and showed that this is both
and necessary and sufficient. This paper is inspired by proof techniques
in the latter paper; an important contribution of our paper is generalizing their rank-1
techniques to accommodate for matrices of arbitrary rank.

\section{Our Approach: Threshold RIP Constant}

Previous work that studied the global optimization landscape of problem
\eqref{p1} typically relied on the restricted isometry property (RIP)
of $\AA$. It is now well-known that if the measurement operator $\AA$
satisfies the restricted isometry property with a sufficiently small
constant $\delta<1/5$ then problem \eqref{p1} contains no spurious
local minima; see Bhojanapalli et al.~\cite{bhojanapalli2016global}. 
\begin{defn}[$\delta$-RIP]
Let $\AA:\R^{n\times n}\to\R^{m}$ be a linear measurement operator.
We say that $\AA$ satisfies the $\delta$-\emph{restricted isometry
property }(or simply $\delta$-RIP) if satisfies the following inequality
\[
(1-\delta)\|M\|_{F}^{2}\le\|\AA(M)\|^{2}\le(1+\delta)\|M\|_{F}^{2}\qquad\forall M\in\M_{2r}
\]
where $\M_{2r}=\{X\in\R^{n\times n}:\rank(X)\le2r\}$ denotes the
set of rank-$2r$ matrices. The RIP constant of $\AA$ is the smallest
value of $\delta$ such that the inequality above holds. 
\end{defn}

Let $\delta\in[0,1)$ denote the RIP constant of $\AA$. It is helpful
to view $\delta$ as a surrogate for the number of measurements $m\ge0$,
with a large value of $\delta$ corresponding a smaller value of $m$
and vice versa. For a wide range of sub-Gaussian measurement ensembles,
if $m\geq C_{0}nr/\delta^{2}$ where $C_{0}$ is an absolute constant,
then $\AA$ satisfies $\delta$-RIP with high probability~\cite{candes2009tight,recht2010guaranteed}.

Take $X\in\R^{n\times r}$ to be a \emph{spurious} point such that
$XX^{T}\ne ZZ^{T}$. Our approach in this paper is to define a \emph{threshold}
number of measurements that would be needed to prevent $X$ from becoming
a local minimum for problem \eqref{A}. Viewing the RIP constant $\delta$
as a surrogate for the number of measurements $m$, we follow a construction
of Zhang et al.~\cite{zhang2019sharp}, and instead define a threshold
$\delta_{\soc}(X)$ on the RIP constant $\delta$ that would prevent
$X$ from becoming a local minimum for problem \eqref{A}. Such a
construction must necessarily take into account all choices of $\AA$
satisfying $\delta$-RIP, including those that adversarially target
$X$, bending the optimization landscape into forming a region of
convergence around the point. On the other hand, such adversarial
choices of $\AA$ must necessarily be defeated for a sufficiently
small threshold on $\delta$, as we already know that spurious local
minima cannot exist for $\delta<1/5$. The statement below makes this
idea precise, and also extends it to a set of spurious points. 
\begin{defn}[Threshold for second-order condition]
Fix $Z\in\R^{n\times r}$. For $X\in\R^{n\times r}$, if $XX^{T}=ZZ^{T}$,
then define $\delta_{\soc}(X)=1$. Otherwise, if $XX^{T}\neq ZZ^{T}$,
then define 
\begin{align}
\delta_{\soc}(X)\equiv & \min_{\mathcal{A}}\{\delta:\nabla f_{\AA}(X)=0,\quad\nabla^{2}f_{\AA}(X)\succeq0,\quad\AA\text{ satisfies }\delta\text{-RIP}\}\label{soc_delta}
\end{align}
where the minimum is taken over all linear measurements $\AA:\mathbb{R}^{n\times n}\to\mathbb{R}^{m}$.
For $\mathcal{W}\subseteq\mathbb{R}^{n\times r},$ define $\delta_{\soc}(\mathcal{W})=\inf_{X\in\mathcal{W}}\delta_{\soc}(X).$ 
\end{defn}

If $\delta<\delta_{\soc}(X)$, then $X$ cannot be a spurious local
minimum by construction, or it would contradict the definition of
$\delta_{\soc}(X)$ as the minimum value. By the same logic, if $\delta<\delta_{\soc}(\mathcal{W})$,
then no choice of $X\in\mathcal{W}$ can be a spurious local minimum.
In particular, it follows that $\delta_{\soc}(\R^{n\times r})$ is
the usual \emph{global} RIP threshold: if $\AA$ satisfies $\delta$-RIP
with $\delta<\delta_{\soc}(\R^{n\times r})$, then $f_{\AA}(X)$ is
guaranteed to admit no spurious local minima. Starting a local optimization
algorithm from any initial point guarantees exact recovery of an $X$
satisfying $XX^{T}=ZZ^{T}$.

Now, suppose we are given an initial point $X_{0}$. It is natural
to measure the \emph{quality} of $X_{0}$ by its relative error, as
in $\varepsilon=\|XX^{T}-ZZ^{T}\|_{F}/\|ZZ^{T}\|_{F}$. If we define
an $\varepsilon$-neighborhood of all points with the same relative
error 
\begin{equation}
\mathcal{B}_{\varepsilon}=\{X\in\mathbb{R}^{n\times r},\|XX^{T}-ZZ^{T}\|_{F}\leq\varepsilon\|ZZ^{T}\|_{F}\}
\end{equation}
then it follows that $\delta_{\soc}(\mathcal{B}_{\varepsilon})$ is
an analogous \emph{local} RIP threshold: if $\AA$ satisfies $\delta$-RIP
with $\delta<\delta_{\soc}(\mathcal{B}_{\varepsilon})$, then $f_{\AA}(X)$
is guaranteed to admit no spurious local minima over all $X\in\mathcal{B}_{\varepsilon}$.
Starting a local optimization algorithm from the initial point $X_{0}$
guarantees either exact recovery of an $X$ satisfying $XX^{T}=ZZ^{T}$,
or termination at a strictly worse point $X$ with $\|XX^{T}-ZZ^{T}\|_{F}>\|X_{0}X_{0}^{T}-ZZ^{T}\|_{F}$.
Imposing further restrictions on the algorithm prevents the latter
scenario from occurring (local strong convexity with gradient descent~\cite{sun2016guaranteed},
strict decrements in the levels set \cite{jain2013low,zhang2019sharp,chi2019nonconvex}),
and so exact recovery is guaranteed.

The numerical difference between the global threshold $\delta_{\soc}(\R^{n\times r})$
and the local threshold $\delta_{\soc}(\mathcal{B}_{\varepsilon})$
is precisely the number of samples that an $\varepsilon$-quality
initial point $X_{0}$ is worth, up to some conversion factor. But
two major difficulties remain in this line of reasoning. First, evaluating
$\delta_{\soc}(X)$ for some $X\in\R^{n\times r}$ requires solving
a minimization problem over the set of $\delta$-RIP operators. Second,
evaluating $\delta_{\soc}(\mathcal{B}_{\varepsilon})$ in turn requires
minimizing $\delta_{\soc}(X)$ over all choices of $X$ within an
$\varepsilon$-neighborhood. Regarding the first point, Zhang~et~al.~\cite{zhang2019sharp}
showed that $\delta_{\soc}(X)$ is the optimal value to a \emph{convex}
optimization problem, and can therefore be evaluated to arbitrary
precising using a numerical algorithm. In the rank-1 case, they solved
this convex optimization in closed-form, and use it to optimize over
all $X\in\mathcal{B}_{\varepsilon}$. Their closed-form solution spanned
9 journal pages, and evoked a number of properties specific to the
rank-1 case (for example, $xy^{T}+yx^{T}=0$ implies $x=0$ and $y=0$,
but $XY^{T}+YX^{T}=0$ may hold for $X\ne0$ and $Y\ne0$). The authors
noted that a similar closed-form solution for the general rank-$r$
case appeared exceedingly difficult. While overall proof technique
is sharp and descriptive, its applicability appears to be entirely
limited to the rank-1 case.

\section{Main results}

In this paper, we bypass the difficulty of deriving a closed-form
solution for $\delta_{\soc}(X)$ altogether by adopting a \emph{sharp}
\textit{\emph{lower-bound}}. This is based on two key insights. First,
a spurious local minimum must also be a spurious critical point, so
the analogous threshold over critical points would give an obvious
lower-bound $\delta_{\foc}(X)\le\delta_{\soc}(X)$. 
\begin{defn}[Threshold for first-order condition]
Fix $Z\in\R^{n\times r}$. For $X\in\R^{n\times r}$, if $XX^{T}=ZZ^{T}$,
then define $\delta_{\foc}(X)=1$. Otherwise, if $XX^{T}\neq ZZ^{T}$,
then define 
\begin{align}
\delta_{\foc}(X)\equiv & \min_{\mathcal{A}}\{\delta:\nabla f_{\AA}(X)=0,\quad\AA\text{ satisfies }\delta\text{-RIP}\},\label{Afoc}
\end{align}
where the minimum is taken over all linear measurements $\AA:\mathbb{R}^{n\times n}\to\mathbb{R}^{m}$.
For $\mathcal{W}\subseteq\mathbb{R}^{n\times r},$ define $\delta_{\foc}(\mathcal{W})=\inf_{X\in\mathcal{W}}\delta_{\foc}(X).$ 
\end{defn}

Whereas the main obstacle in Zhang~et~al.~\cite{zhang2019sharp}
is the considerable difficulty in deriving a closed-form solution
for $\delta_{\soc}(X)$, we show in this paper that it is relatively
straightforward to solve $\delta_{\foc}(X)$ in closed-form, to result
in a simple, geometric solution. 
\begin{thm}
\label{main_cor} Fix $Z\in\R^{n\times r}$. Given $\AA$ satisfying
$\delta$-RIP and $X\in\R^{n\times r}$ such that $XX^{T}\ne ZZ^{T}$,
we have $\delta_{\foc}(X)=\cos\theta$, where 
\begin{equation}
\sin\theta=\|Z^{T}(I-XX^{\dag})Z\|_{F}\;\big/\;\|XX^{T}-ZZ^{T}\|_{F}.\label{sin}
\end{equation}
and $X^{\dag}$ denotes the pseudo-inverse of $X$. It follows that if
$\delta<\cos\theta$, then $X$ is not a spurious critical point of
$f_{\AA}(X)$. If $\delta\ge\cos\theta$, then there exists some $\AA^{\star}$
satisfying $\cos\theta$-RIP such that $\nabla f_{\AA}(X)=0$. 
\end{thm}

The complete proof of Theorem~\ref{main_cor} is given in Appendix A and a sketch is given in section 5. There
is a nice geometric interpretation: the exact value of $\delta_{\foc}(X)$
depends largely on the\emph{ incidence angle} between the column space
of $X$ and the column space of $Z$. When the angle between $XX^{T}$
and $ZZ^{T}$ becomes small, the projection of $XX^{T}$ onto $ZZ^{T}$
becomes large. As a result, $\sin\theta$ becomes small and $\cos\theta$
becomes large. Therefore, Theorem~\ref{main_cor} says that in regions
where $XX^{T}$ and $ZZ^{T}$ are more aligned, fewer samples are
required to prevent $X$ from becoming a spurious critical point.
In regions where $XX^{T}$ and $ZZ^{T}$ are more orthogonal, a larger
sample complexity is needed. Indeed, these are precisely the adversarial
locations for which a large number of samples are required to prevent
spurious local minima from appearing.

\begin{figure}
\centering \includegraphics[width=0.5\textwidth]{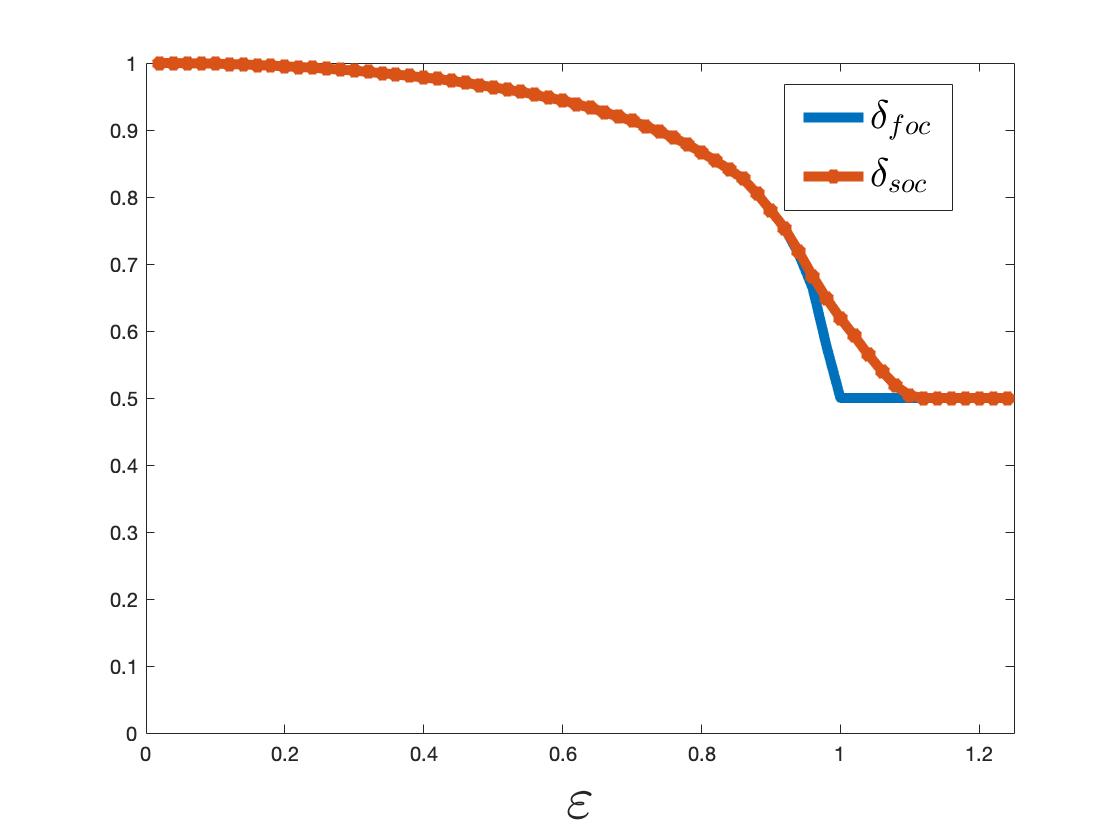} \caption{This paper is motivated by two key insights. First, it is relatively
straightforward to solve $\delta_{\protect\foc}(X)$ in closed-form
(Theorem~\ref{main_cor}). Second, the resulting lower-bound $\delta_{\protect\soc}(X)\ge\max\{\delta_{\protect\foc}(X),\delta^{*}\}$
($\delta^{*}=1/2$ for rank 1 and $\delta^{*}=1/5$ for rank $>1$)
is remarkably tight. This means that $\max\{\delta_{\protect\foc}(\mathcal{B}_{\varepsilon}),\delta^{*}\}$
is a tight lower bound for $\delta_{\protect\foc}(\mathcal{B}_{\varepsilon})$.}
 \label{compare} 
\end{figure}

The lower-bound $\delta_{\foc}(X)\le\delta_{\soc}(X)$ appears conservative,
because critical points should be much more ubiquitous than local
minima over a non-convex landscape. In particular, observe that $\delta_{\foc}(X)=\cos\theta\to0$
as $X\to0$, which makes sense because $X=0$ is a saddle point for
all choices of $\AA$. In other words, for any region $\mathcal{W}$
that contains $0$, the lower-bound becomes trivial, as in $\delta_{\foc}(\mathcal{W})=0<\delta_{\soc}(\mathcal{W})$.
Our second insight here is that we must simultaneously have $\delta_{\soc}(X)\ge1/5$
due to the global threshold of Bhojanapalli et al.~\cite{bhojanapalli2016global}
(or $\delta_{\soc}(x)\ge1/2$ in the rank-1 case due to Zhang et al.~\cite{zhang2019sharp}).
Extending this idea over sets yields the following lower-bound 
\begin{equation}
\delta_{\soc}(\mathcal{W})\ge\max\{\delta_{\foc}(\mathcal{W}),\delta^{*}\}\qquad\text{for all }\mathcal{W}\subseteq\R^{n\times r},\label{eq:central_lb}
\end{equation}
where $\delta^{*}=1/2$ for $r=1$ and $\delta^{*}=1/5>1$. This bound
is \emph{remarkably} tight, as shown in Figure~\ref{compare} for
$\mathcal{W}=\mathcal{B}_{\varepsilon}$ over a range of $\varepsilon$.
Explicitly solving the optimization $\delta_{\foc}(\mathcal{B}_{\varepsilon})=\inf_{X\in\mathcal{B}_{\varepsilon}}\delta_{\foc}(X)$
using Theorem~\ref{main_cor} and substituting into (\ref{eq:central_lb})
yields the following.\footnote{We denote $[x]_{+}=\max\{0,x\}$.} 
\begin{thm}
\label{local} Let $\AA$ satisfy $\delta$-RIP. Then we have $\delta_{\foc}(\mathcal{B}_{\varepsilon})>\sqrt{1-C\varepsilon}$
for all $\epsilon\le1/C$, where $C=\|ZZ^{T}\|_{F}/\sigma_{\min}^{2}(Z)$.
Hence, if 
\begin{equation}
\delta<\max\left\{ \sqrt{[1-C\varepsilon]_{+}},\delta^{*}\right\} 
\end{equation}
where $\delta^{*}=1/2$ if $r=1$ and $\delta^{*}=1/5$ if $r>1$,
then $f_{\AA}(X)$ has no spurious critical point within an $\varepsilon$-neighborhood
of the solution: 
\begin{equation}
\nabla f_{\AA}(X)=0,\quad\|XX^{T}-ZZ^{T}\|_{F}\leq\varepsilon\|ZZ^{T}\|_{F}\quad\Longleftrightarrow\quad XX^{T}=ZZ^{T}.
\end{equation}
\end{thm}

The complete proof of this theorem is in Appendix B. Theorem~\ref{local}
says that the number of samples needed to eliminate spurious critical
points within an $\varepsilon$-neighborhood of the solution decreases dramatically as $\varepsilon$ becomes small. Given
that $m\geq C_{0}nr/\delta^{2}$ sub-Gaussian measurements are needed
to satisfy $\delta$-RIP, we can translate Theorem~\ref{local} into
the following sample complexity bound. 
\begin{cor}
\label{gauss}Let $\AA:\mathbb{R}^{n\times n}\to\mathbb{R}^{m}$ be
a sub-Gaussian measurement ensemble. If 
\[
m\geq\min\left\{ \frac{1}{[1-C\varepsilon]_{+}},25\right\} C_{0}nr
\]
then with high probability there are no spurious local minima within
$\mathcal{B}_{\varepsilon}.$ 
\end{cor}

The proof of Corollary~6 follows immediately from Theorem~\ref{local} combined with the direct relationship between the RIP-property and the sample complexity for sub-Gaussian measurement ensembles. 
We see that the relationship
between the quality of the initial point and the number of samples
saved is essentially \emph{linear}. Improving the quality of the initial point by a linear factor corresponds to a linear decrease in sample complexity. Moreover, the rate of improvement depends on the constant $C$. This shows that in the non-convex setup of
matrix sensing, there is a significant difference between a good initial
point and a mediocre initial point. In the case that $C=\|ZZ^{T}\|_{F}/\sigma_{\min}^{2}(Z)$ is large, this difference is even more pronounced. 
\section{Proof of Main Results}

\subsection{Notation and Definitions}

We use $\|\cdot\|$ for the vector $2$-norm and use $\|\cdot\|_{F}$
to denote the Frobenius norm of a matrix. For two square matrices
$A$ and $B$, $A\succeq B$ means $B-A$ is positive semidefinite.
The trace of a square matrix $A$ is denoted by $\mathrm{tr}(A)$.
The \emph{vectorization} $\vector(A)$ is the length-$mn$ vector
obtained by stacking the columns of $A$. Let $\AA:\R^{n\times n}\to\R^{m}$
be a linear measurement operator, and let $Z\in\R^{n\times r}$ be
a fixed ground truth matrix. We define $\A=[\vector(A_{1}),\ldots,\vector(A_{m})]$
as the matrix representation of $\AA$, and note that $\vector[\AA(X)]=\A\,\vector(X)$.
We define the error vector $\e$ and its Jacobian $\X$ to satisfy
\begin{subequations} \label{eX}
\begin{align}
\e & =\vector(XX^{T}-ZZ^{T})\\
\X\,\vector(Y) & =\vector(XY^{T}+YX^{T})\qquad\text{for all }Y\in\R^{n\times r}.
\end{align}
\end{subequations}

\subsection{Proof Sketch of Theorem 4}

A complete proof of Theorem 4 relies on a few technical lemmas, so
we defer the complete proof to Appendix A. The key insight is that
$\delta_{\foc}(X)$ is the solution to a \textit{convex} optimization
problem, which we can solve in closed-form. At first sight, evaluating
$\delta_{\foc}(X)$ seems very difficult as it involves solving an
optimization problem over the set of $\delta$-RIP operators, as defined
in equation \ref{Afoc} . However, a minor modification of Theorem
8 in Zhang et al. \cite{zhang2019sharp} shows that $\delta_{\foc}(X)$
can be reformulated as a convex optimization problem of the form 
\begin{equation}
\eta(X)\quad\equiv\quad\max_{\eta,\H}\left\{ \eta\quad:\quad\X^{T}\H\e=0,\quad\eta I\preceq\H\preceq I\right\} .\label{eq:prim_1}
\end{equation}
where $\eta(X)$ is related to $\delta_{\foc}(X)$ by 
\begin{equation}
\delta_{\foc}(X)=\frac{1-\eta(X)}{1+\eta(X)}.\label{eq:delta_eta}
\end{equation}

We will show that problem (\ref{eq:prim_1}) actually has a simple
closed-form solution. First, we write its Lagrangian dual as 
\begin{align}
\underset{y,U_{1},U_{2}}{\text{ minimize }}\quad & \tr(U_{2})\label{eq:dual_1}\\
\text{subject to }\quad & (\X y)\e^{T}+\e(\X y)^{T}=U_{1}-U_{2}\nonumber \\
 & \tr(U_{1})=1,\quad U_{1},U_{2}\succeq0.\nonumber 
\end{align}
Notice that strong duality holds because Slater's condition is trivially
satisfied by the dual: $y=0$ and $U_{1}=U_{2}=2I/n(n+1)$ is a strictly
feasible point. It turns out that the dual problem can be rewritten
as an optimization problem over the eigenvalues of the matrix $(\X y)\e^{T}+\e(\X y)^{T}$.
The proof of this in in Appendix A.

For any $\alpha\in\mathbb{R}$ we denote $[\alpha]_{+}=\max\{0,+\alpha\}$
and $[\alpha]_{-}=\max\{0,-\alpha\}$. The dual problem can be written
as 
\[
\min_{y}\frac{\tr{[M(y)]_{-}}}{\tr{[M(y)]_{+}}}=\min_{y}\frac{\sum_{i}{\lambda_{i}[M(y)]_{-}}}{\sum_{i}{\lambda_{i}[M(y)]_{+}}},\quad\mathrm{where}\quad M(y)=(\X y)\e^{T}+\e(\X y)^{T},
\]
and $\lambda_{i}[M(y)]$ denotes the eigenvalues of the rank-2 matrix
$M(y)$. It is easy to verify that the only two non-zero eigenvalues
of $(\X y)\e^{T}+\e(\X y)^{T}$ are 
\[
\|\mathbf{X}y\|\|\e\|\left(\cos\theta_{y}\pm1\right),\quad\text{ where }\cos\theta_{y}=\frac{\e^{T}\mathbf{X}y}{\|\e\|\|\mathbf{X}y\|}.
\]
It follows that 
\[
\eta(X)=\min_{y}\frac{1-\cos\theta_{y}}{1+\cos\theta_{y}}
\]
and therefore 
\[
\delta_{\foc}(X)=\max_{y}\cos\theta_{y}=\max_{y}\frac{\e^{T}\mathbf{X}y}{\|\e\|\|\mathbf{X}y\|}.
\]
Let $y^{*}$ be the optimizer of the optimization problem above, then
$\theta_{y^{*}}$ is simply the incidence angle between the column space of $X$ and the error
vector $\e$. Thus we have $y^{\star}=\arg\min_{y}\|\e-\X y\|.$ Using
Lemma \ref{costheta} in Appendix A, we show that solving for $y^{*}$
yields a closed-form expression for $\theta_{y^{*}}$ in the form
\[
\sin\theta_{y^{*}}=\frac{\|Z^{T}(I-XX^{\dag})Z\|_{F}}{\|XX^{T}-ZZ^{T}\|_{F}}.
\]

Hence we have $\delta_{\foc}(X)=\cos\theta$, with $\theta = \theta_{y^*}$ given
by the equation above.

\subsection{Proof of Theorem \ref{local}}

The proof of Theorem \ref{local} is based on the following lemma.
Its proof is very technical and can be can be found in Appendix B. 
\begin{lem}
Let $Z\ne0$ and suppose that $\|XX^{T}-ZZ^{T}\|_{F}\le\epsilon\|ZZ\|_{F}^{2}$.
Then 
\[
\sin^2\theta=\frac{\|Z^{T}(I-XX^{\dagger})Z\|_{F}^2}{\|XX^{T}-ZZ^{T}\|_{F}^2}\le\frac{\epsilon}{2\sigma_{\min}^{2}(Z)/\|ZZ^{T}\|_{F}-\epsilon}.
\]
\label{sintheta_bd}
\end{lem}
To prove Theorem 5, we simply set $C_{1}=\sigma_{\min}^{2}(Z)/\|ZZ^{T}\|_{F}$
and write
\[
\cos\theta=\sqrt{1-\sin^{2}\theta}\ge\sqrt{1-\frac{\epsilon}{2C_{1}-\epsilon}}.
\]
It is easy to see that $\frac{\epsilon}{2C_{1}-\epsilon}$
is dominated by the linear function $\varepsilon/C_1$ so long as $\epsilon\le C_{1}$. This follows directly from the fact that $\frac{\epsilon}{2C_{1}-\epsilon}$ is convex between $0$ and $C_1$. Thus we have
\[
\cos\theta \ge\sqrt{1-\frac{\epsilon}{C_1}}
\] Since this lower bound holds for all $X$ in $\mathcal{B}_\varepsilon$,  it follows that $\delta_{\foc}(\mathcal{B}_{\varepsilon})\ge\sqrt{1-\varepsilon/C_{1}}$. 

\section{Numerical Results}

In this section we give a geometric interpretation for Theorem \ref{main_cor}, which we already alluded to in section 4: the sample complexity to eliminate spurious critical points is small in regions where the column spaces of $X$ and $Z$ are more aligned and large in regions where they are orthogonal.
We also numerically verify that $\delta_{\foc}(X)$ is
a tight lower bound for $\delta_{\soc}(X)$ for a wide range of $\varepsilon$,
providing numerical evidence that the bound in Theorem \ref{local}
is tight. 
 
Our main results and geometric insights hold for \textit{any rank}, but for ease of visualization
we focus on the rank-1 case where $x$ and $z$ are now just vectors.
To measure the alignment between the column space of $x$ and that of $z$ in the rank-1 case , we define the length ratio and the incidence angle as 
\[
\rho=\frac{\|x\|}{\|z\|},\quad\cos\phi=\frac{x^{T}z}{\|x\|\|z\|}.
\]
 Our goal is to plot how sample complexity depends on this alignment.  Visualizing the dependence of sample complexity on $\rho$ and $\cos\phi$ is particularly easy in rank-1 because these two parameters completely determine the values
of both $\delta_{\foc}(x)$ and $\delta_{\soc}(x)$. See \cite{zhang2019sharp}
section 8.1 for a proof of this fact. This allows us to plot the level
curves of $\delta_{\foc}(x)$ and $\delta_{\soc}(x)$ over the parameter
space $\rho$ and $\phi$ in Figure \ref{levels}. This is shown by
the blue curves. Since we are particularly interested in sample complexity near the ground truth, we also plot the level sets of the function
${\|xx^{T}-zz^{T}\|_{F}}/{\|zz^{T}\|_{F}}$ using red curves. The horizontal axis is
the value of $\rho\cos\phi$ and the vertical axis is the value of
$\rho\sin\phi$.

We can immediately see that in regions in the optimization landscape
where $x$ is more aligned with $z$, i.e., when $\sin\phi$ is small, 
the values of both threshold functions tend to be high and a relatively
small number of samples suffices to prevent $x$ from becoming a spurious
critical point. However, when $x$ and $z$ becomes closer to being
orthogonal, i.e., when $\cos\phi$ is close to $0$, then $\delta_{\foc}(x)$ becomes arbitrarily small, and $\delta_{\soc}(x)$
also becomes smaller, albeit to a lesser extent. As a result, preventing
$x$ from becoming a spurious critical point (or spurious local minima)
in these regions require many more samples. This intuition also permeates to the high-rank case, even though visualization
becomes difficult, and a slightly more general definition of length ratio and alignment is required. 
Similar to the rank-1 case,  in regions where $XX^{T}$
and $ZZ^{T}$ are more aligned, the sample complexity required to
eliminate spurious critical points is small and in regions where $XX^T$ and $ZZ^T$ are close to orthogonal, a small sample complexity is required. 

Regarding the tightness of using $\delta_\foc(X)$ as a lower bound for $\delta_\soc(X)$, note that if we look at the level sets of ${\|xx^{T}-zz^{T}\|_{F}}/{\|zz^{T}\|_{F}}$,
we see that in regions close to the ground truth, both $\delta_{\soc}(x)$
and $\delta_{\foc}(x)$ are very close to $1$. This is in perfect
agreement with our results in Theorem \ref{local}, where we showed
that a small $\varepsilon$ results in a large $\delta_{\foc}(\mathcal{B}_\varepsilon)$.
Moreover, the shapes of the level curves of $\delta_{\soc}$ and $\delta_{\foc}$
that flow through the regions near the ground truth are almost identical.
This indicates that for a large region near the ground truth, the
second-order condition, i.e., the hessian being positive semidefinite,
is inactive. This is the underlying mechanism that causes $\delta_{\foc}$
to be a tight lower bound for $\delta_{\foc}$.

\begin{figure}[h!]
\centering \subfloat[]{ \includegraphics[width=65mm,height=60mm]{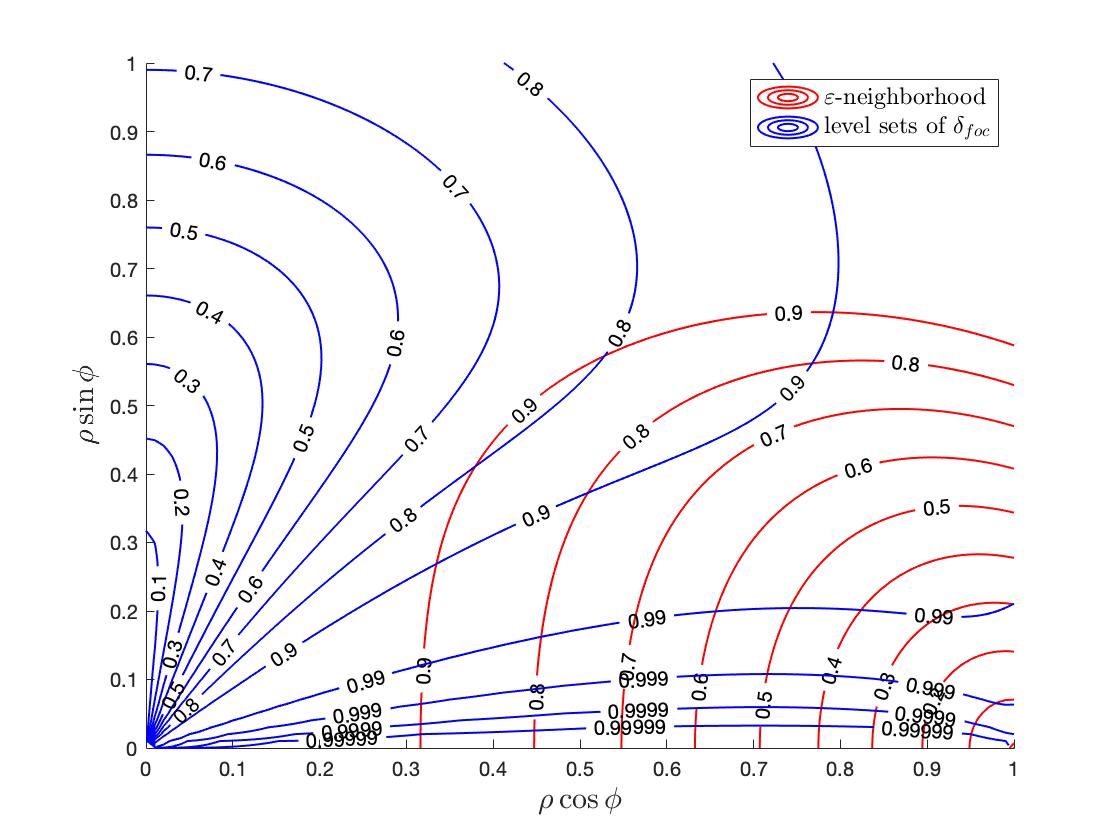}

}\subfloat[]{ \includegraphics[width=65mm,height=60mm]{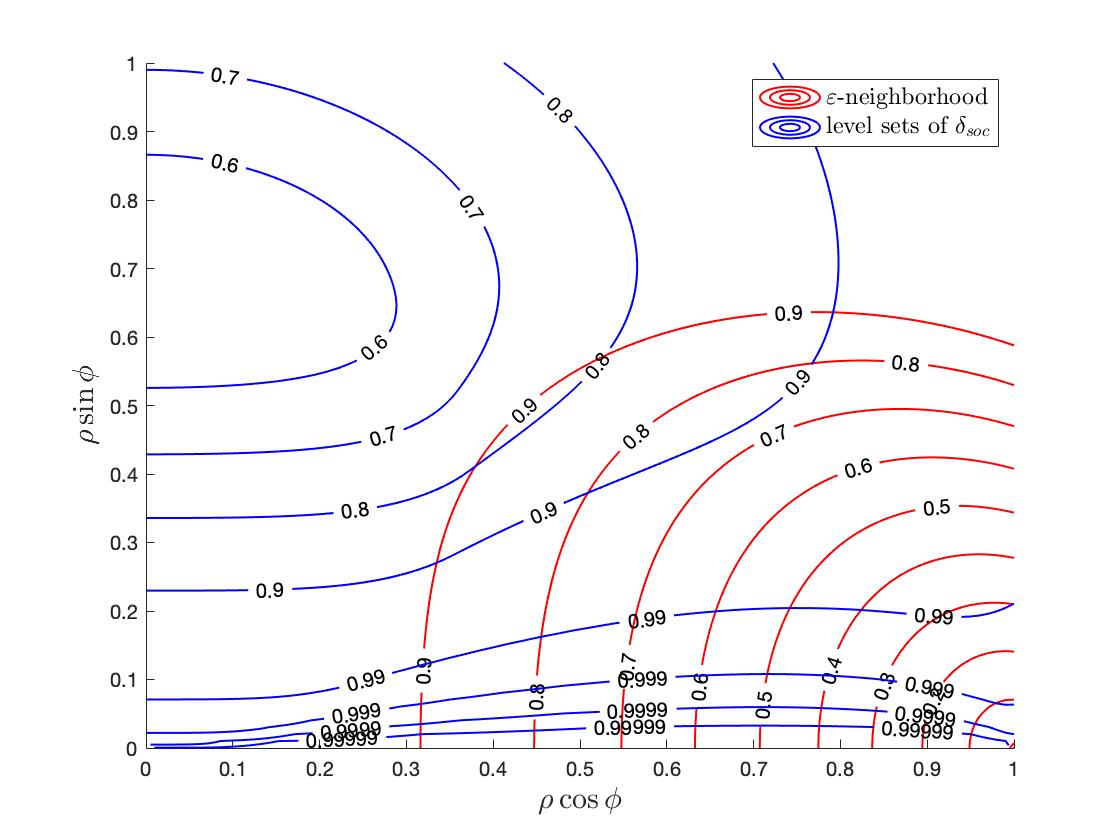}

}\hspace{0mm}

\caption{(a) the level sets of $\delta_{\protect\foc}$ and ${\|xx^{T}-zz^{T}\|_{F}}/{\|zz^{T}\|_{F}}$
(b) the level sets of $\delta_{\protect\soc}$ and ${\|xx^{T}-zz^{T}\|_{F}}/{\|zz^{T}\|_{F}}$}
\label{levels} 
\end{figure}

\section{Conclusions}

Recent work by Bhojanapalli et al. \cite{bhojanapalli2016global}
has shown that the non-convex optimization landscape of matrix sensing
contains no spurious local minima when there are sufficiently large
amount of samples. However, these theoretical bounds on the sample
complexity are very conservative compared to the number of samples
needed in real applications like power state estimation. In our paper,
we provide one explanation for this phenomenon: in real life, we often
have access to good initial points, which can reduce the number of
samples we need. The main results of our paper give a mathematical
characterization of this phenomenon. We define a function $\delta_{\soc}(X)$
that gives a \textit{precise} threshold on the number of samples needed
to prevent $X$ from becoming a spurious local minima. Although $\delta_{\soc}$
is difficult to compute exactly, we obtain a closed-form, sharp lower
bound using convex optimization. As a result, we are able to characterize
the \textit{tradeoff} between the quality of the initial point and
the sample complexity. In particular, we show that  a linear improvement in the quality of the initial
point corresponds to a linear decrease in sample complexity. 

On a more general level, our work uses new techniques to paint a full
picture for the non-convex landscape of matrix sensing: the problem
becomes more ``non-convex'' (requiring more samples to eliminate spurious
local minima) as we get further and further away from the global min.
Once we are sufficiently far away, it becomes necessary to rely on
global guarantees instead. Thus, our work brings new insight into
how a non-convex problem can gradually become more tractable either
through more samples or a better initial point and provides a tradeoff
between these two mechanisms. For future work, it would be interesting
to see if similar techniques can be extended to other non-convex models
such as neural networks.


\section*{Acknowledgements}
Partial financial support was provided by the National Science Foundation under award ECCS-1808859.

\section*{Broader Impact}

Many modern applications in engineering and computer science, and
in machine learning in particular often have to deal with non-convex
optimization. However, many aspects of non-convex optimization are
still not well understood. Our paper provides more insight into the
optimization landscape of a particular problem: low-rank matrix factorization.
In addition, the methods we develop can potentially be used to understand
many other non-convex problems. This is a step towards a more thorough
analysis of current algorithms for non-convex optimization and also
a step towards developing better and more efficient algorithms with
theoretical guarantees.

 \bibliographystyle{plain}
\bibliography{yourbibfile}

\newpage{}

\appendix

\section*{Appendix A.1}
In Appendix A we fill out the missing details in the proof sketch
of Section 5.2 and provide a complete proof of Theorem 4, which we restate below. %
\begin{thm} (Same as theorem 4). 
\label{main_cor} Fix $Z\in\R^{n\times r}$. Given $\AA$ satisfying
$\delta$-RIP and $X\in\R^{n\times r}$ such that $XX^{T}\ne ZZ^{T}$,
we have $\delta_{\foc}(X)=\cos\theta$, where 
\begin{equation}
\sin\theta=\|Z^{T}(I-XX^{\dag})Z\|_{F}\;\big/\;\|XX^{T}-ZZ^{T}\|_{F}.\label{sin}
\end{equation}
and $X^{\dag}$ denotes the pseudo-inverse of $X$. It follows that if
$\delta<\cos\theta$, then $X$ is not a spurious critical point of
$f_{\AA}(X)$. If $\delta\ge\cos\theta$, then there exists some $\AA^{\star}$
satisfying $\cos\theta$-RIP such that $\nabla f_{\AA}(X)=0$. 
\end{thm}

Before we prove the theorem above, we first prove two technical lemmas. The first lemma gives an explicit solution to the eigenvalues of a rank-2 matrix and the second lemma characterizes the solution to an SDP that will be a part of the proof of theorem 4. 

\begin{lem}
Given $a, b \in \mathbb{R}^{n},$ the matrix $M=a b^{T}+b a^{T}$ has eigenvalues $\lambda_{1} \geq \cdots \geq \lambda_{n}$ where:
$$
\lambda_{i}=\left\{\begin{array}{ll}
+\|a\|\|b\|(1+\cos \theta) & i=1 \\
\vspace{-1mm}\\
-\|a\|\|b\|(1-\cos \theta) & i=n \\
\vspace{-1mm}\\
0 & \mathrm{ otherwise }
\end{array}\right.
$$
and $\theta \equiv \arccos \left(\frac{a^{T} b}{\|a\|\|b\|}\right)$ is the angle between $a$ and $b$.
\label{2eig}
\end{lem}

\begin{lem}
Given a matrix $M\neq 0$ we can split the matrix $M$
into a positive and negative part satisfying 
\[
M=M_{+}-M_{-}\quad where\quad M_{+},M_{-}\succeq0,\quad M_{+}M_{-}=0.
\]
Then the following problem has solution 
\[
\min_{{\alpha\in\mathbb{R}\atop U,V\succeq0}}\{\operatorname{tr}(V):\operatorname{tr}(U)=1,\alpha M=U-V\} = \min \left\{ \frac{\operatorname{tr}\left(M_{-}\right)}{\operatorname{tr}\left(M_{+}\right)}, \frac{\operatorname{tr}\left(M_{+}\right)}{\operatorname{tr}\left(M_{-}\right)} \right\}.
\] \label{eigsplit}
\end{lem}
\begin{proof}
	(Lemma \ref{2eig}). Without loss of generality, assume that $\|a\|=\|b\|=1 .$ (Otherwise, we can rescale $\hat{a}=$ $a /\|a\|, \hat{b}=b /\|b\|$ and write $M=\|a\|\|b\| (\hat{a} \hat{b}^{T}+\hat{b} \hat{a}^{T}) $. Now decompose $b$ into a tangent and
normal component with respect to $a,$ as in
$$
b=a \underbrace{a^{T} b}_{\cos \theta}+\underbrace{\left(I-a a^{T}\right) b}_{c \sin \theta}=a \cos \theta+c \sin \theta
$$
where $c$ is a unit normal vector with $\|c\|=1$ and $a^{T} c=0 .$ Thus $ab^T+ba^T$ can be written as 
$$
a b^{T}+b a^{T}=\left[\begin{array}{ll}
a & c
\end{array}\right]\left[\begin{array}{cc}
2 \cos \theta & \sin \theta \\
\sin \theta & 0
\end{array}\right]\left[\begin{array}{ll}
a & c
\end{array}\right]^{T}.
$$
This shows that $M$ is spectrally similar to a $2 \times 2$ matrix with eigenvalues $\cos \theta \pm 1$.
\end{proof}

\begin{proof}
(Lemma \ref{eigsplit}). 
In this proof we will consider two cases: $\tr(M_-) \leq \tr(M_+)$ and $\tr(M_-) \geq \tr(M_+)$. We'll see that in the first case, the optimal value is $\tr(M_-)/\tr(M_+)$ and in the second case, the optimal value is $\tr(M_+)/\tr(M_1)$.

First, assume that $\tr(M_-) \leq \tr(M_+)$. Let $p^{*}$ be the optimal value. Then we have 
\begin{align}
p^{\star} & =\max_{\beta}\min_{{\alpha\in\mathbb{R}\atop U,V\succeq0}}\{\operatorname{tr}(V)+\beta\cdot[1-\operatorname{tr}(U)]:\alpha M=U-V\}\\
 & =\max_{\beta}\min_{\alpha\in\mathbb{R}}\left\{ \beta+\min_{U,V\succeq0}\{\operatorname{tr}(V)-\beta\cdot\operatorname{tr}(U):\alpha M=U-V\}\right\} \nonumber\\
 & =\max_{\beta}\min_{\alpha\in\mathbb{R}}\left\{ \beta+ \min_{U}\left[\operatorname{tr}\left(U-\alpha M\right)-\beta\cdot\operatorname{tr}\left(U\right)\right]:U-\alpha M \succeq 0, U\succeq 0\right\} \nonumber \\
 & =\max_{\beta}\min_{\alpha\in\mathbb{R}}\left\{ \beta + \min_{U} [-\alpha \tr(M)+(1-\beta)\tr(U)]:U-\alpha M \succeq 0, U\succeq 0\right\} \nonumber \\
  & =\max_{\beta \leq 1}\min_{\alpha\in\mathbb{R}}\left\{ \beta + \min_{U} [-\alpha \tr(M)+(1-\beta)\tr(U)]:U-\alpha M \succeq 0, U\succeq 0\right\}. 
\end{align}
Note that the first line converts the equality constraint into a Lagrangian. The second line simply rearranges the terms. The third line plugs in $V=U-\alpha M$. The fourth line again rearranges the terms. The last line follows from the observation that if $\beta>1$, then the inner minimization over $U$ will go to negative infinity since the trace of $U$ can be arbitrarily large. 

First, consider the case $\alpha\geq 0$. Then we have $\alpha M = \alpha M_{+}-\alpha M_{-}$. Since $1-\beta \geq 0$, the minimization over $U$ is achieved at $U=\alpha M_{+}$. Plugging this value into the optimization problem, then (19) becomes 
\[
\max_{\beta \leq 1} \min_{\alpha \geq 0} \{\beta +\alpha [\tr(M_{-})-\beta\tr(M_{+})]  \}
\]
If $\tr(M_{-})-\beta\tr(M_{+}) < 0$, then the optimal value of the inner minimization will go to negative infinity. On the other hand, if $\tr(M_{-})-\beta\tr(M_{+}) \geq 0$ then the minimum inside is achieved at $\alpha=0$. Thus the problem above is equivalent to
\[
\max_{\beta \leq 1} \{\beta:  \tr(M_{-})-\beta\tr(M_{+}) \geq 0\}.
\]
Since $\tr(M_{-}) \leq \tr(M_{+})$, the optimal value of the problem above is achieved at $\tr(M_{-})/\tr(M_{+}) \leq 1$.  
Now suppose that $\alpha \leq 0$. Then the optimal value for $U$ is achieved at $U=-\alpha M_{-}$. Plugging this value in and (19) becomes
\[
\max_{\beta \leq 1} \min_{\alpha \leq 0} \{\beta +\alpha [\beta \tr(M_{-})-\tr(M_{+})]  \}.
\]
Similar to before, we must have $\beta \tr(M_{-})-\tr(M_{+}) \leq 0$, so $\beta \leq \tr(M_{+})/\tr(M_{-})$. Since $\tr(M_{-}) \leq \tr(M_{+})$, the optimal value in this case is just $\beta = 1$. Combining the results for $\alpha \geq 0$ and $\alpha \leq 0$, we find that when $\tr(M_{-}) \leq \tr(M_{+})$, the optimal value is 
\[
p^* = \min \left\{1, \frac{\tr(M_{-})}{\tr(M_{+})} \right\} = \frac{\tr(M_{-})}{\tr(M_{+})}.
\]
Repeating the same arguments for when $\tr(M_{-}) \geq \tr(M_{+})$,  we see that in this case the optimal value becomes 
\[
p^*  = \min\left\{\frac{\tr(M_{+})}{\tr(M_{-})}, 1 \right\} = \frac{\tr(M_{+})}{\tr(M_{-})}.
\]
Finally, combining these two cases, i.e., $\tr(M_{-}) \geq \tr(M_{+})$ and $\tr(M_{-}) \leq \tr(M_{+})$, we obtain 
\[
p^* = \min \left\{ \frac{\operatorname{tr}\left(M_{-}\right)}{\operatorname{tr}\left(M_{+}\right)}, \frac{\operatorname{tr}\left(M_{+}\right)}{\operatorname{tr}\left(M_{-}\right)} \right\},
\]
which completes the proof.

\end{proof}

\section*{Appendix A.2}
Now we are ready to prove Theorem 4. Recall that the first order threshold function is defined as the solution to the following optimization problem:
\[
\delta_{\foc}(X)\equiv  \min_{\mathcal{A}}\{\delta:\nabla f_{\AA}(X)=0,\quad\AA\text{ satisfies }\delta\text{-RIP}\}
\]
Using Theorem 8 from \cite{zhang2019sharp}, the optimization problem above can be formulated as \begin{equation}
\eta(X)\quad\equiv\quad\max_{\eta,\H}\left\{ \eta\quad:\quad\X^{T}\H\e=0,\quad\eta I\preceq\H\preceq I\right\} .\label{eq:prim_1}
\end{equation}
where $\eta = (1-\delta_\foc)/(1+\delta_\foc)$. Our goal is to solve this optimization problem 
in closed form. 
In Section 5.2, we wrote the dual of problem (\ref{eq:prim_1}) as
\begin{align}
\underset{y,U_{1},U_{2}}{\mathrm{min}}\quad & \tr(U_{2})\label{eq:dual_1}\\
\mathrm{subject~to}\quad & (\X y)\e^{T}+\e(\X y)^{T}=U_{1}-U_{2}\nonumber \\
 & \tr(U_{1})=1,\quad U_{1},U_{2}\succeq0.\nonumber 
\end{align}
and stated that this dual problem can be
rewritten as an optimization problem over the eigenvalues of a rank-2
matrix. This is given in the lemma below. To simplify notation,
here we define a positive/negative splitting: for any $\alpha\in\mathbb{R}_{+}$
we denote $[\alpha]_{+}=\max\{0,+\alpha\}$ and $[\alpha]_{-}=\max\{0,-\alpha\}$.
This idea can be extended to matrices by applying splitting to the
eigenvalues. 
\begin{lem}
Given data $\e$ and $\X\neq 0$, define 
\begin{align}
\eta=\underset{y,U_{1},U_{2}}{\mathrm{min}}\quad & \tr(U_{2})\label{eq:dual_1}\\
\mathrm{subject~to}\quad & (\X y)\e^{T}+\e(\X y)^{T}=U_{1}-U_{2}\nonumber \\
 & \tr(U_{1})=1,\quad U_{1},U_{2}\succeq0.\nonumber 
\end{align}
Define $M(y)$ to be the rank-2 matrix $(\X y)\e^{T}+\e(\X y)^{T}$ and let $\lambda_{i}[M(y)]$ denote its eigenvalues.
Then $\eta$ can be evaluated as
\[
\eta=\min_{y\neq 0}\frac{\tr{[M(y)]_{-}}}{\tr{[M(y)]_{+}}}=\min_{y\neq 0}\frac{\sum_{i}{\lambda_{i}[M(y)]_{-}}}{\sum_{i}{\lambda_{i}[M(y)]_{+}}}  = \min_{y\neq 0}\frac{1-\cos\theta_{y}}{1+\cos\theta_{y}},
\]
where 
$\cos\theta_{y}= {\e^{T}\mathbf{X}y}/{\|\e\|\|\mathbf{X}y\|}.$
 \label{dual}
\end{lem}

The proof of Lemma \ref{dual} relies mainly on the two lemmas we proved in the preceding section. 
\begin{proof}
(Lemma \ref{dual}).  Let $y = \alpha\hat{y}$, where $\|\hat{y}\|=1$ and $\alpha \in \mathbb{R}^n$. Thus the optimization problem (\ref{eq:dual_1}) becomes 
\begin{align}
\eta=\underset{\alpha,\hat{y}, U_{1},U_{2}}{\mathrm{min}}\quad & \tr(U_{2})\nonumber\\
\mathrm{subject~to}\quad & \alpha\cdot [(\X \hat{y})\e^{T}+\e(\X \hat{y})^{T}]=U_{1}-U_{2}\nonumber \\
 & \tr(U_{1})=1, \quad \|\hat{y}\|=1, \quad U_{1},U_{2}\succeq0.\nonumber 
\end{align}
To solve this problem, first we keep $\hat{y}$ fixed, and optimize over $\alpha, U_1,U_2$. This gives us the problem
\begin{align}
\underset{\alpha, U_{1},U_{2}}{\mathrm{min}}\quad & \tr(U_{2})\nonumber\\
\mathrm{subject~to}\quad & \alpha\cdot [(\X \hat{y})\e^{T}+\e(\X \hat{y})^{T}]=U_{1}-U_{2}\nonumber \\
 & \tr(U_{1})=1, \quad U_{1},U_{2}\succeq0.\nonumber 
\end{align}
Notice that if we set $M(\hat{y}) = (\X \hat{y})\e^{T}+\e(\X \hat{y})^{T}$, then the problem above is in  exactly the same form as the one in lemma \ref{eigsplit}. Therefore, its optimal value is 
\[
 \min \left\{ \frac{\operatorname{tr}\left(M(\hat{y})_{-}\right)}{\operatorname{tr}\left(M(\hat{y})_{+}\right)}, \frac{\operatorname{tr}\left(M(\hat{y})_{+}\right)}{\operatorname{tr}\left(M(\hat{y})_{-}\right)} \right\}.
\]
Finally, to obtain $\eta$, we still need to optimize over $\hat{y}$, i.e.,
\[
\eta = \min_{\|\hat{y}\|=1} \min \left\{ \frac{\operatorname{tr}\left(M(\hat{y})_{-}\right)}{\operatorname{tr}\left(M(\hat{y})_{+}\right)}, \frac{\operatorname{tr}\left(M(\hat{y})_{+}\right)}{\operatorname{tr}\left(M(\hat{y})_{-}\right)} \right\}.
\]
Since both the numerator and the denominator are linear in $y$, we can ignore the constraint $\|\hat{y}\|=1$ and simply optimize over $y$, which gives us
\[
\eta = \min_{y \neq 0} \min \left\{ \frac{\operatorname{tr}\left(M({y})_{-}\right)}{\operatorname{tr}\left(M({y})_{+}\right)}, \frac{\operatorname{tr}\left(M({y})_{+}\right)}{\operatorname{tr}\left(M({y})_{-}\right)} \right\}.
\]
With lemma \ref{2eig}, we see that the only two eigenvalues of $M(y)$ are
\[
\|\X y\|\|y\|(\cos\theta_y+1), \qquad \|\X y\|\|y\|(\cos\theta_y-1),
\]
where $\cos\theta_{y}= {\e^{T}\mathbf{X}y}/{\|\e\|\|\mathbf{X}y\|}.$
It follows that $\tr(M_{-}) = \|\X y\|\|y\|(1-\cos\theta_y)$ and $\tr(M_{+}) = \|\X y\|\|y\|(\cos\theta_y+1)$. Thus
\[
\eta = \min_{y\neq 0}  \min\left\{\frac{1-\cos\theta_y}{1+\cos\theta_y}, \frac{1+\cos\theta_y}{1-\cos\theta_y}\right\}.
\]
Notice that in the optimization problem above, if the minimum is achieved at some $y^*$, it must also be achieved at $-y^*$, due to symmetry. Therefore, it suffices to optimize over only the first term $\frac{1-\cos\theta_y}{1+\cos\theta_y}$, so we get  
 \[
\eta =\min_{y\neq 0} \frac{1-\cos\theta_y}{1+\cos\theta_y}.
\]
This completes the proof. 
\end{proof}

Notice that Lemma \ref{dual} reduces problem \ref{eq:prim_1} to only depend on the values of $\cos\theta_y$. Now, to complete the proof of Theorem 4, we just need one additional lemma that gives a closed form solution for $\cos\theta_y$, which we state below. 

\begin{lem}
\label{closecos}
Let $X,Z$ be $n\times r$ matrices of any rank, and define $\e$
and $\X \neq 0$ as in equations \ref{eX}(a) and \ref{eX}(b). Then, the incidence angle
$\theta$ between $\e$ and $\range(\X)$, defined as in 
\[
\cos\theta=\max_{y\neq 0}\left\{ \frac{\e^{T}\X y}{\|\e\|\|\X y\|}\right\} =\frac{\|\X\X^{\dag}\e\|}{\|\e\|},
\]
has closed-form expression 
\[
\sin\theta=\frac{\|Z^{T}(I-XX^{\dag})Z\|_{F}}{\|XX^{T}-ZZ^{T}\|_{F}}
\]
where $X^{\dag}$ denotes the Moore--Penrose pseudoinverse of $X$.
\label{costheta} \end{lem}
\begin{proof}(Lemma \ref{closecos}).
Define $y^{\star}=\arg\min_{y}\|\e-\X y\|$ and decompose $\e=\X y^{\star}+w$.
The optimality condition for $y^{\star}$ reads $\X^{T}(\e-\X y^{\star})=\X^{T}w=0$,
so we substitute $\e^T\X = (y^*)^T\X^T\X$ to yield
\[
\|\e\|\cos\theta=\|\e\|\max_{y\neq 0}\left\{ \frac{\e^{T}\X y}{\|\e\|\|\X y\|}\right\} =\max_{y\neq 0} \left\{ \frac{(y^{\star})^{T}\X^{T}\X y}{\|\X y\|}\right\} =\|\X y^{\star}\|,
\]
and therefore $\|\e\|\sin\theta=\|w\|=\min_{y}\|\e-\X y\|$, because we have $\e = \X y^*+w$ with $w^T \X y^*=0$. Now,
define $Q=\orth(X)\in\R^{n\times q}$ where $q=\rank(X)\le r$, and
define $P\in\R^{n\times(n-q)}$ as the orthogonal complement of $Q$.
Decompose $X=Q\hat{X}$, and $Z=Q\hat{Z}_{1}+P\hat{Z}_{2}$, and note
that 
\begin{align*}
\|w\| & =\min_{y}\|\e-\X y\|\\
 & =\min_{Y}\|(XX^{T}-ZZ^{T})-(XY^{T}+YX^{T})\|_{F}\\
 & =\min_{[\hat{Y}_{1};\hat{Y}_{2}]\in\R^{n\times r}}\left\Vert \begin{bmatrix}\hat{X}\hat{X}^{T}-\hat{Z}_{1}\hat{Z}_{1}^{T} & -\hat{Z}_{1}\hat{Z}_{2}^{T}\\
-\hat{Z}_{2}\hat{Z}_{1}^{T} & -\hat{Z}_{2}\hat{Z}_{2}^{T}
\end{bmatrix}-\begin{bmatrix}\hat{X}\hat{Y}_{1}^{T}+\hat{Y}_{1}\hat{X}^{T} & \hat{X}\hat{Y}_{2}^{T}\\
\hat{Y}_{2}\hat{X}^{T} & 0
\end{bmatrix}\right\Vert _{F}\\
 & =\|\hat{Z}_{2}\hat{Z}_{2}^{T}\|_{F}
\end{align*}
From the second line to the third, we apply a change of basis onto $[Q~ P]$, which preserves the Frobenius norm. To derive the last line, notice that the $q\times r$ matrix $\hat{X}$
has full row rank, so that $\hat{X}\hat{X}^{T}\succ0$ and $\hat{X}\hat{X}^{\dag}=I_{q}$. We want to show that there exists $\hat{Y}_1$ such that 
\[
\hat{X}\hat{Y}_{1}^{T}+\hat{Y}_{1}\hat{X}^{T} = \hat{X}\hat{X}^{T}-\hat{Z}_{1}\hat{Z}_{1}^{T}.
\]
Since the right hand side is symmetric, we can write it as $L+L^T$, where $L$ is some lower-triangular matrix. Thus it suffices to show that there exists $\hat{Y}_1$ such that $\hat{X}\hat{Y}_{1}^{T} = L$, which follows from that fact that $\hat{X}$ has full row-rank. Similarly, there exists some $\hat{Y}_{2}$ such that $\hat{X}\hat{Y}_2 = -\hat{Z}_{2}\hat{Z}_{1}^{T}$. Thus, all terms except the last one cancels out and we are left with $\min_{y}\|\e-\X y\|=\|\hat{Z}_{2}\hat{Z}_{2}^{T}\|_{F}$.

Finally, note that $Q\hat{Z}_{1}=XX^{\dag}Z$ and $P\hat{Z}_{2}=(I-XX^{\dag})Z$
and that 
\begin{align*}
\|\hat{Z}_{2}\hat{Z}_{2}^{T}\|_{F}^{2} & =\|P\hat{Z}_{2}\hat{Z}_{2}^{T}P^{T}\|_{F}^{2}\\
 & =\|(I-XX^{\dag})ZZ^{T}(I-XX^{\dag})\|_{F}^{2}\\
 & =\tr[(I-XX^{\dag})ZZ^{T}(I-XX^{\dag})ZZ^{T}(I-XX^{\dag})]\\
 & =\tr[Z^{T}(I-XX^{\dag})ZZ^{T}(I-XX^{\dag})Z]\\
 & =\|Z^{T}(I-XX^{\dag})Z\|_{F}^{2}.
\end{align*}
Substituting the definition of $\e$ completes the proof. 
\end{proof}

Now theorem 4 will be a direct consequence of  lemma \ref{dual} and lemma \ref{costheta}.  We give a proof below. 
\begin{proof} (Theorem 4).
Note that $\delta_\foc$ is related to $\eta$ by the equation
\[
\eta=\frac{1-\delta_\foc}{1+\delta_\foc}.
\]
Applying lemma \ref{dual}, we immediately get
\[
\delta_\foc(X) = \max_{y\neq 0} \cos\theta_y = \max_{y\neq 0} \frac{\e^{T}\mathbf{X}y}{\|\e\|\|\mathbf{X}y\|}.
\]

From lemma \ref{costheta}, we see that this optimization problem over $y$ has a simple closed form solution of the form
\[
\delta_\foc(X) = \cos\theta, \quad \text{ where } \sin\theta=\frac{\|Z^{T}(I-XX^{\dag})Z\|_{F}}{\|XX^{T}-ZZ^{T}\|_{F}}.
\] 
This completes the proof. 
\end{proof}

\section*{Appendix B}

In this section we provide a complete proof of Theorem 5, which includes
all the intermediate calculations that was skipped in Section 5.3.
We begin by proving a bound on $\sin\theta$.
\begin{lem}[Same as Lemma \ref{sintheta_bd}]
Let $Z\ne0$ and suppose that $\|XX^{T}-ZZ^{T}\|_{F}\le\epsilon\|ZZ\|_{F}^{2}$.
Then 
\[
\sin^2\theta=\frac{\|Z^{T}(I-XX^{\dagger})Z\|_{F}^2}{\|XX^{T}-ZZ^{T}\|_{F}^2}\le\frac{\epsilon}{2(\sigma_{\min}^{2}(Z)/\|ZZ^{T}\|_{F})-\epsilon}.
\]
\end{lem}

\begin{proof}
The problem is homogeneous to scaling $X\gets\alpha X$ and $Z\gets\alpha Z$
for the same $\alpha$; Since $Z\ne0$, we may rescale $X$ and $Z$
until $\|ZZ\|_{F}^{2}=1$. Additionally, we can assume that
\[
X=\begin{bmatrix}X_{1}\\
0
\end{bmatrix}\qquad Z=\begin{bmatrix}Z_{1}\\
Z_{2}
\end{bmatrix}\qquad\text{ where }X_{1},Z_{1}\in\mathbb{R}^{r\times r},Z_2 \in \mathbb{R}^{(n-r)\times r}
\]
due to the rotational invariance of the problem. (Concretely, we compute
the QR decomposition $QR=[X,Z]$ with $Q\in\R^{n\times2r}$ noting
that $X=QQ^{T}X$ and $Z=QQ^{T}Z$. We then make a change of basis
$X\gets Q^{T}X$ and $Z\gets Q^{T}Z$). Then, observe that
\begin{equation}
\|Z^{T}(I-XX^{\dagger})Z\|_{F}=\left\Vert \begin{bmatrix}Z_{1}\\
Z_{2}
\end{bmatrix}^{T}\left(I-\begin{bmatrix}I & 0\\
0 & 0
\end{bmatrix}\right)\begin{bmatrix}Z_{1}\\
Z_{2}
\end{bmatrix}\right\Vert _{F}=\|Z_{2}^{T}Z_{2}\|_{F}=\|Z_{2}Z_{2}^{T}\|_{F}\label{eq:numer}
\end{equation}
and that $\|Z_{2}Z_{2}^{T}\|_{F}^{2}\le\epsilon^{2}$ because
\begin{align}
\|XX^{T}-ZZ^{T}\|_{F}^{2} & =\left\Vert \begin{bmatrix}Z_{1}Z_{1}^{T}-X_{1}X_{1}^{T} & Z_{1}Z_{2}^{T}\\
Z_{2}Z_{1in }^{T} & Z_{2}Z_{2}^{T}
\end{bmatrix}\right\Vert _{F}^{2}\nonumber \\
 & =\|Z_{1}Z_{1}^{T}-X_{1}X_{1}^{T}\|_{F}^{2}+2\|Z_{1}Z_{2}^{T}\|_{F}^{2}+\|Z_{2}Z_{2}^{T}\|_{F}^{2}\le\epsilon^{2}.\label{eq:err_bnd}
\end{align}
In order to derive a non-vacuous bound, we will need to lower-bound
the term $\|Z_{1}Z_{2}^{T}\|_{F}^{2}$ as follows
\begin{equation}
\|Z_{1}Z_{2}^{T}\|_{F}^{2}=\mathrm{tr}(Z_{1}^{T}Z_{1}Z_{2}^{T}Z_{2})\ge\lambda_{\min}(Z_{1}^{T}Z_{1})\mathrm{tr}(Z_{2}^{T}Z_{2})=\sigma_{\min}^{2}(Z_{1})\|Z_{2}\|_{F}^{2}.\label{eq:Z12bnd}
\end{equation}
To lower-bound $\sigma_{\min}^{2}(Z_{1})$, observe that
\[
A+B\succeq\mu I\iff A\succeq\mu I-B\succeq(\mu-\|B\|_2)I,
\]
and therefore 
\begin{align}
\sigma_{\min}^{2}(Z_{1})=\lambda_{\min}(Z_{1}^{T}Z_{1}) & \ge\lambda_{\min}(Z_{1}^{T}Z_{1}+Z_{2}^{T}Z_{2})-\lambda_{\max}(Z_{2}^{T}Z_{2})\nonumber \\
 & =\sigma_{\min}^{2}(Z)-\|Z_{2}Z_{2}^{T}\|_2\ge\sigma_{\min}^{2}(Z)-\|Z_{2}Z_{2}^{T}\|_{F}\nonumber \\
 & \ge\sigma_{\min}^{2}(Z)-\epsilon.\label{eq:sigmin_bnd}
\end{align}
Finally, we substitute (\ref{eq:numer}) and (\ref{eq:err_bnd}) and
perform a sequence of reductions:
\begin{align*}
\frac{\|Z^{T}(I-XX^{\dagger})Z\|_{F}^2}{\|XX^{T}-ZZ^{T}\|_{F}^2} & =\frac{\|Z_{2}Z_{2}^{T}\|_{F}^{2}}{\|Z_{1}Z_{1}^{T}-X_{1}X_{1}^{T}\|_{F}^{2}+2\|Z_{1}Z_{2}^{T}\|_{F}^{2}+\|Z_{2}Z_{2}^{T}\|_{F}^{2}}\\
 & \overset{\text{(a)}}{\le}\frac{\|Z_{2}Z_{2}^T\|_{F}^{2}}{2\|Z_{1}Z_{2}^{T}\|_{F}^{2}+\|Z_{2}Z_{2}^{T}\|_{F}^{2}}\overset{\text{(b)}}{\le}\frac{\|Z_{2}Z_{2}^T\|_{F}^{2}}{2\sigma_{\min}^2(Z_{1})\|Z_{2}\|_{F}^{2}+\|Z_{2}Z_{2}^{T}\|_{F}^{2}}\\
 & \overset{\text{(c)}}{\le}\frac{\|Z_{2}Z_{2}^T\|_{F}\|Z_{2}\|_{F}^{2}}{2\sigma_{\min}^2(Z_{1})\|Z_{2}\|_{F}^{2}+\|Z_{2}Z_{2}^{T}\|_{F}\|Z_{2}\|_{F}^{2}}=\frac{\|Z_{2}Z_{2}^T\|_{F}}{2\sigma_{\min}^2(Z_{1})+\|Z_{2}Z_{2}^{T}\|_{F}}\\
 & \overset{\text{(d)}}{\le}\frac{\epsilon}{2(\sigma_{\min}^{2}(Z)-\epsilon)+\epsilon} \leq \frac{\epsilon}{2\sigma_{\min}^{2}(Z)-\epsilon}.
\end{align*}
Step (a) sets $X_{1}=Z_{1}$ to minimize the denominator; step (b)
bounds $\|Z_{1}Z_{2}^{T}\|_{F}^{2}$ using (\ref{eq:Z12bnd}); step
(c) bounds $\|Z_{2}Z_{2}^T\|_{F}\le\|Z_{2}\|_{F}^{2}$ noting that a
function like $x/(1+x)$ is increasing with $x$; step (d) substitutes
$\|Z_{2}Z_{2}\|_{F}\le\epsilon$ and $\sigma_{\min}^{2}(Z_{1})\ge\sigma_{\min}^{2}(Z)-\epsilon$. 
\end{proof}


\end{document}